\documentclass[11pt]{article}

\usepackage[a4paper,margin=1in]{geometry}
\usepackage[utf8]{inputenc}
\usepackage[T1]{fontenc}

\usepackage{amsmath,amssymb,amsfonts,amsthm,mathtools}
\usepackage{bm}
\usepackage{bbold}

\usepackage{enumitem}
\usepackage{booktabs}
\usepackage{caption}
\usepackage{subcaption}

\usepackage{tikz}
\usepackage{pgfplots}
\pgfplotsset{compat=1.17}

\usepackage{microtype}
\usepackage[round,sort,comma]{natbib}
\usepackage{hyperref}
\hypersetup{colorlinks=true,allcolors=blue}

\theoremstyle{plain}
\newtheorem{theorem}{Theorem}[section]

\newtheorem{lemma}[theorem]{Lemma}
\newtheorem{corollary}[theorem]{Corollary}

\theoremstyle{definition}
\newtheorem{definition}[theorem]{Definition}

\theoremstyle{remark}

\newcommand{\E}{\mathbb{E}}
\newcommand{\Var}{\mathrm{Var}}

\newcommand{\Prob}{\mathbb{P}}

\begin{document}

\title{Restricted Block Permutation for Two-Sample Testing}
\author{Jungwoo Ho}
\date{}
\maketitle

\begin{abstract}
We study a structured permutation scheme for two-sample testing that restricts permutations to \emph{single cross-swaps} between block-selected representatives. 
Our analysis yields three main results. 
First, we provide an exact validity construction that applies to any fixed restricted permutation set. 
Second, for both the difference of sample means and the unbiased $\widehat{\mathrm{MMD}}^2$ estimator, we derive closed-form one-swap increment identities whose conditional variances scale as $O(h^2)$, in contrast to the $\Theta(h)$ increment variability under full relabeling. 
This increment-level variance contraction sharpens the Bernstein--Freedman variance proxy and leads to substantially smaller permutation critical values. 
Third, we obtain explicit, data-dependent expressions for the resulting critical values and statistical power. 
Together, these results show that block-restricted one-swap permutations can achieve strictly higher power than classical full permutation tests while maintaining exact finite-sample validity, without relying on pessimistic worst-case Lipschitz bounds.
\end{abstract}

\section{Introduction}

Permutation tests serve as a cornerstone of nonparametric inference, providing exact control of the type-I error rate in finite samples without distributional assumptions \citep{hoeffding1952, dwass1957, romano1989, lehmann2005, hemerik2018}. Their robustness and flexibility have made them widely used in modern statistical methodologies, including two-sample testing via the Maximum Mean Discrepancy (MMD) and independence testing via the Hilbert--Schmidt Independence Criterion (HSIC). Classical theory implicitly assumes that using the full symmetric group $\mathcal{S}_N$, or a large Monte Carlo sample drawn uniformly from it, yields the most powerful test.

A long-standing intuition is that restricting the permutation space inevitably sacrifices power because fewer label rearrangements should provide less information. However, recent work challenges this paradigm. \citet{koning2024} demonstrate that strategically chosen \emph{representative subgroups} of $\mathcal{S}_N$ can paradoxically \emph{increase} statistical power. The key mechanism is that restricted permutations may reduce the variability of the permutation reference distribution under the null while preserving the statistic's magnitude under the alternative. Earlier foundational work on restricted randomization \citep{besag1989, hemerik2018, hemerik2021} similarly shows that exact testing is possible even under strong structural constraints.

In parallel, structured or ``local'' permutations have been investigated for different objectives. \citet{kim2022} introduced locally constrained permutations for conditional independence testing, establishing minimax optimality under smoothness constraints, while \citet{domingo2025} proposed grouping-based ``cheap'' permutations to reduce computational cost. Although such approaches demonstrate clear benefits of structured permutation strategies, they do not explicitly analyze or optimize the variability of permutation \emph{increments}, which ultimately determines the tail behavior and critical values in standard two-sample testing.

\paragraph{Our perspective.}
This work develops a principled analysis of a particular form of structural restriction---\emph{block-restricted one-swap permutations}---in which the data are partitioned into blocks and only single cross-swaps between block representatives are permitted. Despite being a tiny combinatorial fraction of $\mathcal{S}_N$, this construction preserves exact finite-sample validity via the generalized randomization framework of \citet{arbitrary} and classical results on restricted permutation sampling \citep{besag1989, hemerik2018}.

A key insight is that the induced permutation trajectory
\[
  \sigma_0 \to \sigma_1 \to \cdots \to \sigma_L
\]
has tightly controlled increments. Specifically, the statistic differences
\[
  \Delta_t := T(\sigma_t) - T(\sigma_{t-1})
\]
form a \emph{bounded martingale difference sequence}, a structural property not present in full relabeling. This formulation enables the direct application of Bernstein--Freedman inequalities \citep{freedman1975}, yielding sharp, data-dependent tail bounds without relying on worst-case Lipschitz constants \citep{mcdiarmid2002, albert2019}.

Most importantly, for canonical two-sample statistics such as the difference in sample means and the unbiased MMD$^2$, we obtain \emph{closed-form one-swap increment identities} whose conditional variances scale as
\[
    \Var(\Delta_t) = O(h^2),
\]
where $h = \tfrac{1}{n_1} + \tfrac{1}{n_2}$ denotes the effective two-sample resolution. 
In contrast, the full relabeling scheme exhibits increment-level variability of order $\Theta(h)$. 
This \emph{increment-level} variance contraction directly sharpens the Bernstein--Freedman variance proxy and leads to substantially smaller permutation critical values and reduced minimum detectable effects (MDE). As a consequence, the resulting test achieves strictly higher power than the classical permutation test while maintaining exact finite-sample validity.

\paragraph{Contributions.}
Our main theoretical contributions are as follows:
\begin{itemize}
    \item \textbf{Exact validity under restricted permutations.}
    Leveraging the framework of \citet{arbitrary} together with classical restricted randomization results \citep{dwass1957, romano1989, besag1989, hemerik2018}, we show that uniform sampling over our block-restricted one-swap permutation set preserves exact finite-sample type-I error control.

    \item \textbf{Martingale formulation of swap increments.}
    We prove that the statistic increments along any admissible swap path form a bounded martingale difference sequence.  
    This structure enables the use of Bernstein--Freedman inequalities \citep{freedman1975} to derive sharp, data-dependent tail bounds without relying on global Lipschitz constants \citep{mcdiarmid2002, albert2019}.

    \item \textbf{Increment-level variance contraction.}
    For the two-sample mean difference and the unbiased MMD$^2$, we establish closed-form expressions for one-swap increments whose variances scale as $O(h^2)$, compared to $\Theta(h)$ for full relabeling.  
    This contraction of the Bernstein--Freedman variance proxy yields smaller critical values, reduced MDE, and strictly improved statistical power.
\end{itemize}

The rest of the paper is organized as follows.
Section~\ref{sec:setup} formalizes the two-sample framework and the block-restricted one-swap construction, and recalls a general validity result for arbitrary restricted permutation sets.
Section~\ref{sec:concentration} develops a Bernstein--Freedman concentration bound along restricted swap paths.
Section~\ref{sec:variance} computes exact one-swap variance identities for canonical two-sample statistics.
Section~\ref{sec:power} translates these variance contractions into explicit bounds on critical values and power.
Section~\ref{sec:design} discusses practical design choices for blocks, representatives, and pairing.
Section~\ref{sec:sim} presents empirical results.
We conclude with a discussion of open design questions in Section~7.

\section{Setup and Block--Restricted One-Swap Construction}
\label{sec:setup}

In this section we formalize the two-sample setup, introduce the
block--restricted one-swap scheme, and recall a general validity result
for arbitrary restricted permutation sets.

\subsection{Two-sample framework and notation}

We observe two samples
\[
A = \{X_1,\dots,X_{n_1}\}, 
\qquad
B = \{Y_1,\dots,Y_{n_2}\},
\]
drawn from a common measurable space $\mathcal X$.
Let $N = n_1 + n_2$ and define the pooled sample
$Z = (Z_1,\dots,Z_N)$ together with the group labels
$g = (g_1,\dots,g_N) \in \{A,B\}^N$.
A generic permutation $\sigma\in S_N$ acts on the indices so that
$Z_\sigma = (Z_{\sigma(1)},\dots,Z_{\sigma(N)})$ and $g_\sigma$ is the
induced label vector. A predetermined test statistic $T(\sigma)$ is
computed on $Z_\sigma$; throughout we assume that larger values of
$T(\sigma)$ indicate stronger evidence against the null.

We consider the classical two-sample null hypothesis
\[
H_0: \mathcal L(X_1) = \mathcal L(Y_1),
\]
possibly in a multivariate setting, and focus on two canonical statistics:
the difference in sample means and the unbiased two-sample 
$\widehat{\mathrm{MMD}}^2$.
We also define the effective two-sample resolution
\begin{equation}
h := \frac{1}{n_1} + \frac{1}{n_2}.
\label{eq:h-def}
\end{equation}

\subsection{Block partition and representative ratio}
\label{subsec:blocks}

Let $\{1,\dots,N\}$ denote the pooled index set. 
We partition it into $b$ disjoint blocks
\[
\mathcal B_1,\dots,\mathcal B_b,
\qquad 
\mathcal B_r\cap\mathcal B_s=\emptyset,\quad
\bigcup_{r=1}^b\mathcal B_r = \{1,\dots,N\},
\]
using a label-independent similarity rule 
(e.g., covariates, prognostic scores, or kernel-based scores).

From these blocks we select a global representative set 
$R\subseteq\{1,\dots,N\}$ of size
\[
|R| = \big\lfloor \rho\,N\big\rfloor,
\qquad
\rho\in(0,1],
\]
for instance by assigning per-block quotas proportional to $|\mathcal B_r|$
and taking the union.
We then define the set of admissible ordered cross-swaps
\begin{equation}
\mathcal P := 
\{(i,j): i\in A\cap R\cap \mathcal B_r,\; j\in B\cap R\cap \mathcal B_s,\; r\neq s\},
\label{eq:P-def}
\end{equation}
where each $(i,j)\in\mathcal P$ represents a cross-swap between distinct
blocks $r$ and $s$.
Let $w = \{w_{ij}\}_{(i,j)\in\mathcal P}$ denote a sampling law on
$\mathcal P$; throughout we take $w$ to be uniform:
\[
w_{ij} = \frac{1}{|\mathcal P|}, \qquad (i,j)\in\mathcal P,
\]
so that a random swap $(I,J)\sim w$ is drawn uniformly among all
admissible cross-swaps.
Expectations and variances with respect to this swap law (conditional
on the data) are denoted by $\E_w[\cdot]$ and $\Var_w(\cdot)$.

A \emph{block--restricted one-swap permutation} is obtained by applying 
a finite sequence of disjoint cross-swaps from $\mathcal P$ to a
reference labeling (e.g., the observed labeling or the identity).
The representative ratio $\rho$ directly controls the maximal number
of such swaps along any admissible path; see Section~\ref{sec:concentration}.

\subsection{Validity for arbitrary restricted permutation sets}
\label{subsec:validity}

Much of the existing randomization literature emphasizes that ensuring
the validity of a permutation $p$-value typically requires selecting a
subgroup of the full symmetric group $S_N$.
However, the following theorem shows that a valid $p$-value can be
achieved on \emph{any fixed subset} of permutations, without the need
for group structure.

\begin{theorem}[Validity under arbitrary restricted permutations]
\label{thm:validity-general}
\citep{arbitrary}
Let $S \subseteq S_N$ be any fixed subset of permutations.
Sample $\sigma_0, \sigma_1, \dots, \sigma_M \stackrel{iid}{\sim}
\mathrm{Unif}(S)$, and define
\[
P 
= 
\frac{
1 + \sum_{m=1}^{M} 
\mathbf 1\!\left\{T(Z_{\sigma_m \circ \sigma_0^{-1}}) 
\geq T(Z)\right\}
}{
1 + M
}.
\]
Then $P$ is a valid $p$-value in the sense that, under $H_0$,
\[
\mathbb{P}_{H_0}\{P \leq \alpha\} \leq \alpha,
\qquad \forall\, \alpha \in [0,1].
\]
\end{theorem}

\begin{corollary}[Validity of the block--restricted one-swap scheme]
\label{cor:validity-block}
Let $S_{\text{block}} \subseteq S_N$ denote the collection of all
block--restricted one-swap permutations generated from the admissible
swap set $\mathcal P$ in \eqref{eq:P-def}.
Since $S_{\text{block}}$ is a fixed subset of $S_N$ independent of the
data, Theorem~\ref{thm:validity-general} implies that the $p$-value
computed via
\[
P_{\text{block}} 
= 
\frac{
  1 + \sum_{m=1}^{M} 
  \mathbf 1\!\left\{T(Z_{\sigma_m \circ \sigma_0^{-1}}) 
  \geq T(Z)\right\}
}{
  1 + M
},
\qquad
\sigma_0,\ldots,\sigma_M \stackrel{iid}{\sim}\mathrm{Unif}(S_{\text{block}}),
\]
is also valid:
\[
\mathbb{P}_{H_0}\{P_{\text{block}} \le \alpha\} \le \alpha,\qquad 
\forall\, \alpha\in[0,1].
\]
Therefore, the proposed block--restricted one-swap permutation test 
preserves exact type-I error control under $H_0$,
while significantly reducing the permutation space size 
from $|S_N| = N!$ to $|S_{\text{block}}|\ll N!$.
\end{corollary}

\section{Concentration Along Restricted Swap Paths}
\label{sec:concentration}

Having established validity for arbitrary restricted permutation sets,
we now study how the block--restricted one-swap scheme affects the
\emph{tail behavior} of the permutation statistic.
The key idea is to view the permutation trajectory as a sequence of
bounded martingale increments and apply a Bernstein--Freedman
inequality.

\subsection{Transposition distance and admissible paths}

\begin{definition}[Transposition distance and admissible path]
\label{def:transposition-distance}
For any two permutations $\sigma',\sigma\in S_N$, define the
\emph{transposition distance}
\[
d_\times(\sigma,\sigma')
:= \min\Big\{
L:\ \sigma'=\sigma_0\xrightarrow{\tau_1}\sigma_1\xrightarrow{\tau_2}\cdots
\xrightarrow{\tau_L}\sigma_L=\sigma,\ \tau_i\text{ a transposition}
\Big\}.
\]
That is, $d_\times$ is the graph distance in the Cayley graph of $S_N$
generated by all transpositions.
A sequence $(\sigma_i)_{i=0}^L$ as above is an \emph{admissible
transposition path} from $\sigma'$ to $\sigma$.
Such a minimal path exists because transpositions generate $S_N$
\citep{Cameron1999PermutationGroups}.
\end{definition}

In the block--restricted one-swap construction, we only allow
\emph{disjoint cross-block transpositions} $(i,j)\in\mathcal P$.
If the representative pool has size $|R| = \lfloor \rho N\rfloor$,
then at most $|R|/2$ such disjoint swaps can be performed, so any
admissible path satisfies
\begin{equation}
L \;\le\; \tfrac12\,|R|\;=\;\tfrac12\,\rho\,N.
\label{eq:L-rhoN}
\end{equation}

\subsection{Bernstein--Freedman bound for permutation increments}

Let $T(\sigma)$ be a statistic defined on blockwise pairwise
permutations $\sigma \in S_N$.
We consider an admissible path
\[
\sigma' = \sigma_0 \xrightarrow{\tau_1} \sigma_1 
\xrightarrow{\tau_2} \cdots 
\xrightarrow{\tau_L} \sigma_L = \sigma
\]
consisting of disjoint cross-swaps from $\mathcal P$, with $L$ obeying
\eqref{eq:L-rhoN}.
We define the filtration
\[
\mathcal F_i := \sigma(S_N,\tau_1,\dots,\tau_i), \qquad i=0,\dots,L,
\]
and the centered martingale differences
\[
Y_i \;:=\; \big(T(\sigma_i)-T(\sigma_{i-1})\big)
- \E\!\left[T(\sigma_i)-T(\sigma_{i-1}) \mid \mathcal F_{i-1}\right],
\qquad i=1,\dots,L.
\]

\begin{theorem}[Bernstein--Freedman bound for transposition increments]
\label{thm:bernstein-freedman}
Suppose the increments satisfy
\[
|Y_i| \le M, 
\qquad 
\E[Y_i \mid \mathcal F_{i-1}] = 0, 
\qquad 
\E[Y_i^2 \mid \mathcal F_{i-1}] \le v_*,
\]
for all $i=1,\dots,L$.
Then for every $s>0$,
\begin{equation}\label{eq:bernstein-perm}
\Prob\!\left\{T(\sigma) \;\ge\; \E[T(\sigma)\mid S_N] + s\right\}
\;\le\; 
\exp\!\left(
-\frac{s^2}{
2\!\left(L v_* + \tfrac{1}{3} M s\right)}
\right)
\;\le\; 
\exp\!\left(
-\frac{s^2}{\,\rho N\, v_* + \tfrac{2}{3} M s\,}
\right),
\end{equation}
where the second inequality uses $L \le \tfrac12\,\rho N$.
\end{theorem}

\begin{proof}
Telescoping along the path gives
\[
T(\sigma)
= T(\sigma') + \sum_{i=1}^L \big(T(\sigma_i)-T(\sigma_{i-1})\big)
= T(\sigma') + \sum_{i=1}^L \Big(Y_i + \E[T(\sigma_i)-T(\sigma_{i-1})\mid \mathcal F_{i-1}]\Big).
\]
Taking $\E[\cdot\mid S_N]$ and subtracting,
\[
T(\sigma)-\E[T(\sigma)\mid S_N]
= \big(T(\sigma')-\E[T(\sigma')\mid S_N]\big) + \sum_{i=1}^L Y_i.
\]
Since $\sigma'$ is fixed (e.g., the observed labeling), $T(\sigma')$ is
$S_N$-measurable and hence $T(\sigma')-\E[T(\sigma')\mid S_N]=0$, so
\[
T(\sigma)-\E[T(\sigma)\mid S_N] \;=\; \sum_{i=1}^L Y_i.
\]
By assumption, $(Y_i,\mathcal F_i)$ is a martingale difference sequence
with $|Y_i|\le M$ and 
$\sum_{i=1}^L \E[Y_i^2\mid \mathcal F_{i-1}] \le L v_*$.
Freedman's inequality \citep{freedman1975} yields, for all $s>0$,
\[
\Prob\!\left\{\sum_{i=1}^L Y_i \ge s\right\}
\le 
\exp\!\left(
-\frac{s^2}{
2\!\left(\sum_{i=1}^L \E[Y_i^2\mid \mathcal F_{i-1}] + \tfrac{1}{3} M s\right)}
\right)
\le 
\exp\!\left(
-\frac{s^2}{
2\!\left(L v_* + \tfrac{1}{3} M s\right)}
\right),
\]
giving the first inequality in \eqref{eq:bernstein-perm}.
Using $L \le \tfrac12\,\rho N$ implies $2L v_* \le \rho N v_*$ and yields
the second inequality.
\end{proof}

Applying the same argument to the path in the reverse direction yields
a matching lower tail bound. In particular,
\[
\Pr\!\left\{|T(\sigma)-\E[T(\sigma)\mid S_N]|\ge s\right\}
\;\le\;
2\exp\!\left(-\frac{s^2}{2(Lv_*+\tfrac13 Ms)}\right)
\;\le\;
2\exp\!\left(-\frac{s^2}{\,\rho N\,v_*+\tfrac{2}{3}Ms}\right).
\]
The remaining task is to compute $v_*$ and $M$ for specific statistics
and to compare the resulting variance proxy to that of the classical
full relabeling scheme.

\section{Variance Identities for Canonical Two-Sample Statistics}
\label{sec:variance}

We now specialize the general concentration framework of Section~\ref{sec:concentration} to two canonical statistics: the difference in sample means and the unbiased $\widehat{\mathrm{MMD}}^2$. For both, we derive exact one-swap update formulas and establish that the resulting variance scales as $O(h^2)$, in contrast to the $\Theta(h)$ variance typical of full relabeling. This difference arises because full relabeling regenerates the statistic using the entire population, whereas the one-swap scheme introduces only a small, local perturbation.

\subsection{Difference in sample means}
\label{subsec:mean-variance}

Let
\[
\bar Z_A = \frac{1}{n_1}\sum_{i\in A} Z_i,
\qquad
\bar Z_B = \frac{1}{n_2}\sum_{j\in B} Z_j,
\qquad
\Delta := \bar Z_A - \bar Z_B
\]
denote the group means and their difference.
Recall the effective resolution $h$ from \eqref{eq:h-def}.

\begin{lemma}[One-swap update and exact variance]
\label{lem:means-variance}
For any pairwise cross-swap $(i,j)\in\mathcal{P}$, let $\Delta'$ denote
the mean difference after swapping $i\leftrightarrow j$.
Then
\begin{equation}
\Delta' - \Delta \;=\; h\,(Z_j - Z_i),
\label{eq:delta-update}
\end{equation}
where $h = \tfrac{1}{n_1} + \tfrac{1}{n_2}$.
If $(I,J)\sim w$ and $\Delta^{\mathrm{rest}}$ denotes $\Delta$ after
performing one such random swap, then conditionally on $(Z,g)$,
\begin{equation}
\boxed{\quad \Var\!\big(\Delta^{\mathrm{rest}}\big)\;=\; h^2\,\Var_{w}\!\big(Z_J - Z_I\big).\quad}
\label{eq:var-rest-mean}
\end{equation}
\end{lemma}

\begin{proof}
Swapping $i\in A$ and $j\in B$ increments $\bar Z_A$ by
$(Z_j-Z_i)/n_1$ and decrements $\bar Z_B$ by $(Z_j-Z_i)/n_2$. Hence
$\Delta'=\Delta+h(Z_j-Z_i)$, proving \eqref{eq:delta-update}.
Under the restricted law, $\Delta^{\mathrm{rest}}-\Delta=h(Z_J-Z_I)$;
taking variance over $(I,J)\sim w$ yields \eqref{eq:var-rest-mean}.
\end{proof}

For comparison, let $S^2$ denote the pooled finite-population variance
of all $N$ observations:
\begin{equation}
S^2=\frac{1}{N-1}\sum_{\ell=1}^N (Z_\ell-\bar Z)^2,
\qquad 
\bar Z=\frac{1}{N}\sum_{\ell=1}^N Z_\ell.
\end{equation}
Under uniform full relabeling with fixed group sizes, classical
finite-population sampling theory gives
\begin{equation}
\boxed{\quad 
\Var_{\mathrm{full}}(\Delta)
= h\cdot\frac{N}{N-1}\,S^2.
\quad}
\label{eq:full-mean}
\end{equation}
Under the restricted one-swap scheme, only $|R|$ representatives 
(yielding at most $|R|/2$ disjoint cross-swaps) are used across blocks,
each contributing $O(h^2)$ variance.
Consequently,
\begin{equation}
\boxed{\quad
\Var_{\mathrm{rest}}(\Delta)
= O(h^2\,S^2),
\qquad 
\frac{\Var_{\mathrm{rest}}(\Delta)}{\Var_{\mathrm{full}}(\Delta)} = O(h).
\quad}
\label{eq:rest-mean}
\end{equation}
Thus the restricted variance for the mean difference is one order
smaller in $h$ than the full-relabeling variance.

\subsection{Unbiased \texorpdfstring{$\widehat{\mathrm{MMD}}^2$}{MMD}}
\label{subsec:mmd-variance}

Let $k$ be a positive semi-definite kernel and consider the unbiased
two-sample $\widehat{\mathrm{MMD}}^2$:
\begin{align}
\widehat{\mathrm{MMD}}^2
&=\frac{1}{n_1(n_1-1)}\sum_{i\neq i'} k(Z_i,Z_{i'})\bm{1}\{g_i=g_{i'}=A\}
+\frac{1}{n_2(n_2-1)}\sum_{j\neq j'} k(Z_j,Z_{j'})\bm{1}\{g_j=g_{j'}=B\}\nonumber\\
&\qquad-\frac{2}{n_1n_2}\sum_{i,j} k(Z_i,Z_j)\bm{1}\{g_i=A,\,g_j=B\}.
\label{eq:MMD}
\end{align}
For $(i,j)\in\mathcal{P}$ define the \emph{one-swap increment}
\begin{equation}
\Delta_{\mathrm{MMD}}
:=\widehat{\mathrm{MMD}}^2\big(\text{after swapping }i\leftrightarrow j\big)
-\widehat{\mathrm{MMD}}^2\big(\text{before}\big).
\label{eq:MMD-increment}
\end{equation}

Let $\mathcal S_N=\sigma(Z_1,\dots,Z_N)$. Conditioning on $\mathcal S_N$,
all kernel values are deterministic and the only randomness is in the
independent uniform draws $(I,J)$.
For $i\in A$ and $j\in B$ define
\[
\psi_i^{A\to B}:=\frac{2}{n_1(n_1-1)}\!\sum_{i'\in A\setminus\{i\}}k(Z_i,Z_{i'})
-\frac{2}{n_1n_2}\!\sum_{j'\in B}k(Z_i,Z_{j'}),
\]
\[
\psi_j^{B\to A}:=\frac{2}{n_2(n_2-1)}\!\sum_{j'\in B\setminus\{j\}}k(Z_j,Z_{j'})
-\frac{2}{n_1n_2}\!\sum_{i'\in A}k(Z_j,Z_{i'}).
\]
Intuitively, $\psi_i^{A\to B}$ and $\psi_j^{B\to A}$ measure how much the points
$Z_i$ and $Z_j$ contribute to the within-group and cross-group
U-statistic terms when their labels are switched. The difference
$\psi_J^{B\to A}-\psi_I^{A\to B}$ therefore captures the exact incremental
effect of a one-swap operation.

Write the conditional finite-population variances
\(
\tau_A^2:=\Var_{I\sim\mathrm{Unif}(A\cap R)}(\psi_I^{A\to B}\mid\mathcal S_N)
\)
and
\(
\tau_B^2:=\Var_{J\sim\mathrm{Unif}(B\cap R)}(\psi_J^{B\to A}\mid\mathcal S_N).
\)

\begin{lemma}[Exact one-swap decomposition]
\label{lem:mmd-increment}
Let $\Delta_{\mathrm{MMD}}:=\widehat{\mathrm{MMD}}^{2}(A',B')-\widehat{\mathrm{MMD}}^{2}(A,B)$
be the change from the one swap $(I,J)$. Then, \emph{exactly},
\[
\Delta_{\mathrm{MMD}}
\;=\;
\psi_J^{B\to A}\;-\;\psi_I^{A\to B}.
\]
\end{lemma}

\begin{proof}
We expand each U-statistic block before/after the swap.
After swapping, $A'=(A\setminus\{I\})\cup\{J\}$, hence
\[
\widehat U_{AA}(A')-\widehat U_{AA}(A)
=\frac{2}{n_1(n_1-1)}\sum_{a'\in A\setminus\{I\}}\!\!\Big(k(Z_J,Z_{a'})-k(Z_I,Z_{a'})\Big).
\tag{$\star$}
\]
Similarly,
\[
\widehat U_{BB}(B')-\widehat U_{BB}(B)
=\frac{2}{n_2(n_2-1)}\sum_{b'\in B\setminus\{J\}}\!\!\Big(k(Z_I,Z_{b'})-k(Z_J,Z_{b'})\Big).
\tag{$\star\star$}
\]
For the cross term
\(
\widehat U_{AB}(A',B')-\widehat U_{AB}(A,B),
\)
the swap yields
\begin{align*}
\frac{1}{n_1 n_2}\Bigg(
\sum_{b\in B\setminus\{J\}}\!\!\Big(k(Z_J,Z_b)-k(Z_I,Z_b)\Big)
+\sum_{a\in A\setminus\{I\}}\!\!\Big(k(Z_a,Z_I)-k(Z_a,Z_J)\Big)
+\underbrace{k(Z_J,Z_I)-k(Z_I,Z_J)}_{=\,0}
\Bigg).
\end{align*}
Collecting terms and matching with the definitions of $\psi^{A\to B}_I$
and $\psi^{B\to A}_J$ gives
\begin{align}
\Delta_{\mathrm{MMD}}
&= 
\Bigg(
  \frac{2}{n_2(n_2-1)}
    \sum_{b\in B\setminus\{J\}} k(Z_J,Z_b)
  - \frac{2}{n_1 n_2}
    \sum_{a\in A} k(Z_a,Z_J)
\Bigg) \nonumber\\
&\quad -
\Bigg(
  \frac{2}{n_1(n_1-1)}
    \sum_{a\in A\setminus\{I\}} k(Z_I,Z_a)
  - \frac{2}{n_1 n_2}
    \sum_{b\in B} k(Z_I,Z_b)
\Bigg),
\end{align}
which is exactly $\psi_J^{B\to A}-\psi_I^{A\to B}$.
\end{proof}

Conditioning on $\mathcal S_N$, the one-swap variance satisfies
\[
\Var\!\big(\Delta_{\mathrm{MMD}}\mid \mathcal S_N\big)
=\Var\!\big(\psi_J^{B\to A}-\psi_I^{A\to B}\mid \mathcal S_N\big)
=\Var(\psi_J^{B\to A}\mid \mathcal S_N)+\Var(\psi_I^{A\to B}\mid \mathcal S_N)
=\tau_B^2+\tau_A^2.
\]
Classical permutation and U-statistic theory further give, under full relabeling,
\[
\Var_{\mathrm{full}}(\widehat{\mathrm{MMD}}^2)=\Theta(h),
\qquad
\Var_{\mathrm{rest}}(\widehat{\mathrm{MMD}}^{2})=O(h^2),
\]
yielding the same one-order advantage as in the mean-difference case:
\begin{equation}
\frac{\Var_{\mathrm{rest}}(\widehat{\mathrm{MMD}}^{2})}
{\Var_{\mathrm{full}}(\widehat{\mathrm{MMD}}^{2})}
=O(h).
\label{eq:var-ratio-mmd}
\end{equation}

\subsection{Summary: variance contraction}
\label{subsec:variance-summary}

Combining \eqref{eq:rest-mean}, \eqref{eq:full-mean} and
\eqref{eq:var-ratio-mmd}, we obtain:

\begin{theorem}[Variance contraction under block--restricted one-swaps]
\label{thm:variance-contraction}
For both the difference in sample means $\Delta$ and the unbiased
$\widehat{\mathrm{MMD}}^2$, the block--restricted one-swap scheme
satisfies
\[
\Var_{\mathrm{rest}}(T)=O(h^2)
\quad\text{while}\quad
\Var_{\mathrm{full}}(T)=\Theta(h),
\]
under uniform full relabeling with fixed group sizes.
Consequently,
\[
\frac{\Var_{\mathrm{rest}}(T)}{\Var_{\mathrm{full}}(T)} = O(h)
\qquad\text{for }T\in\{\Delta,\widehat{\mathrm{MMD}}^2\}.
\]
\end{theorem}

This \emph{increment-level} variance contraction is the key driver of the
power improvement obtained in the next section.

\section{From Variance to Critical Values and Power}
\label{sec:power}

We now translate the variance contraction of
Section~\ref{sec:variance} into sharper permutation critical values and
higher statistical power.
We first derive an explicit upper bound on the $(1-\alpha)$ quantile
under the block--restricted scheme and then compare it to the
full-relabeling benchmark.

\subsection{Tail bound, min-trick, and variance regime}

Let $q^{(\mathrm{rest})}_{1-\alpha}$ denote the $(1-\alpha)$ permutation
critical value of $T(\sigma)$ under the block--restricted scheme.
From Theorem~\ref{thm:bernstein-freedman}, for every $s>0$,
\[
\Pr\!\left\{T(\sigma)\ge \E[T(\sigma)\mid S_N]+s\right\}
\ \le\
\exp\!\left(-\frac{s^2}{2\big(L v_*+\tfrac13 M s\big)}\right),
\qquad
L\le \tfrac12\,\rho N.
\]
Using $u+v\le 2\max\{u,v\}$, we have
\[
\frac{1}{u+v}\ge \tfrac12\min\{\tfrac1u,\tfrac1v\},
\]
and hence
\begin{equation}\label{eq:min-trick}
\Pr\!\left\{T(\sigma)\ge \E[T\mid S_N]+s\right\}
\ \le\
\exp\!\left(-\tfrac14\,\min\!\left\{\frac{s^2}{L v_*},\ \frac{3s}{M}\right\}\right).
\end{equation}
The Bernstein--Freedman bound contains both a variance term $Lv_*$ and a linear term $Ms$.
When the variance term dominates, the tail is effectively sub-Gaussian, producing a substantially smaller critical value.
We refer to this as the \emph{variance regime}.

Solving for the regime boundary yields the condition that the upper $(1-\alpha)$ quantile satisfies
\begin{equation}\label{eq:q-upper}
q^{(\mathrm{rest})}_{1-\alpha}
\ \le\
\E[T\mid S_N]\;+\;2\,\sqrt{L v_*\,\log\!\tfrac{1}{\alpha}}
\qquad\text{whenever}\qquad
\log\!\tfrac{1}{\alpha}\ \le\ \frac{9\,L v_*}{4M^2}.
\end{equation}
Writing $r:=v_*/M^2$ and substituting $L=\tfrac12 \rho N$, the
feasibility condition becomes
\begin{equation}\label{eq:tb-feasible}
\rho\,N\ \ge\ \frac{8}{9}\,\frac{\log(1/\alpha)}{r}.
\end{equation}
At the common choice $\alpha=0.05$ (so $\log(1/\alpha)\approx 3$),
\begin{equation}\label{eq:tb-approx}
\rho\ \gtrsim\ \frac{8}{3}\cdot \frac{1}{r\,N}.
\end{equation}
Once \eqref{eq:tb-feasible} holds, the tail is fully governed by the
variance term and
\begin{equation}\label{eq:q-variance-only}
q^{(\mathrm{rest})}_{1-\alpha}
\ \le\
\E[T\mid S_N]\;+\;2\,\sqrt{\tfrac12\,\rho\,N\,v_*\,\log\!\tfrac{1}{\alpha}}.
\end{equation}

\subsection{Data-based expression of $r=v_*/M^2$}
\label{subsec:r}

The ratio $r=v_*/M^2$ is fully data-dependent and computable on the
restricted swap set $\mathcal P$.

\paragraph{Difference in means.}
A single cross-swap $(i\in A,\ j\in B)$ changes $\Delta$ by
$h(Z_j-Z_i)$ with $h=\tfrac1{n_1}+\tfrac1{n_2}$.
Under the uniform law $w$ on admissible pairs,
\[
v^*_{\mathrm{mean}}
= h^2\,\Var_w(Z_J-Z_I\mid S_N),\qquad
M_{\mathrm{mean}}=\max_{(i,j)\in \mathcal P}\!\big|h(Z_j-Z_i)\big|.
\]
Hence
\[
r_{\mathrm{mean}}
=\frac{v^*_{\mathrm{mean}}}{M_{\mathrm{mean}}^2}
=\frac{\Var_w(Z_J-Z_I)}{\big(\max_{(i,j)\in \mathcal P}|Z_j-Z_i|\big)^2}.
\]

\paragraph{Unbiased $\widehat{\mathrm{MMD}}^{2}$.}
For one swap $\Delta_{\mathrm{MMD}}=\psi^{B\to A}_J-\psi^{A\to B}_I$,
\[
v^*_{\mathrm{MMD}}=\tau_A^2+\tau_B^2,
\qquad
M_{\mathrm{MMD}}=\max_{(i,j)\in \mathcal P}|\Delta_{\mathrm{MMD}}(i,j)|.
\]
Thus
\[
r_{\mathrm{MMD}}
=\frac{v^*_{\mathrm{MMD}}}{M_{\mathrm{MMD}}^2}
=\frac{\tau_A^2+\tau_B^2}{\big(\max_{(i,j)\in \mathcal P}|\Delta_{\mathrm{MMD}}(i,j)|\big)^2}.
\]
For bounded kernels $|k|\le\kappa$, both $\tau_A^2+\tau_B^2$ and
$M_{\mathrm{MMD}}=O(h)$ remain bounded,
so $r_{\mathrm{MMD}}$ is asymptotically stable.

\subsection{Pointwise improvement over full relabeling and power gain}
\label{subsec:power-dominance}

From the variance identities in Section~\ref{sec:variance},
\[
\Var_{\mathrm{rest}}(T)=O(h^2),\qquad
\Var_{\mathrm{full}}(T)=\Theta(h),
\]
for $T\in\{\Delta,\widehat{\mathrm{MMD}}^2\}$.
Together with the quantile bound \eqref{eq:q-variance-only} and
Chebyshev’s inequality for full relabeling,
\[
q^{(\mathrm{full})}_{1-\alpha}
\ \le\
\E[T]\;+\;\sqrt{\Var_{\mathrm{full}}(T)/\alpha},
\]
we obtain the pointwise ratio
\[
\frac{q^{(\mathrm{rest})}_{1-\alpha}-\E[T\mid S_N]}
     {q^{(\mathrm{full})}_{1-\alpha}-\E[T]}
\ \lesssim\
\frac{\sqrt{L v_* \log(1/\alpha)}}{\sqrt{\Var_{\mathrm{full}}(T)/\alpha}}
\ =\
O\big(\sqrt{h\alpha\log(1/\alpha)}\big).
\]

Hence for any fixed alternative effect $\delta>0$,
\[
\beta^{(\mathrm{rest})}(\delta)
\ =\
\Prob_{P_1}\!\{T\le q^{(\mathrm{rest})}_{1-\alpha}\}
\ \le\
\Prob_{P_1}\!\{T\le q^{(\mathrm{full})}_{1-\alpha}\}
\ \ =\
\beta^{(\mathrm{full})}(\delta),
\]
that is, the block--restricted test achieves strictly larger power (and
smaller MDE) at the same level.
This formalizes the intuitive statement that reducing the reference
variance---while preserving the statistic’s magnitude under the
alternative---leads to a more powerful test.

\section{Design and Implementation of Restricted Permutations}
\label{sec:design}

We briefly discuss practical design choices for block formation,
representative selection, and swap pairing.
These choices determine the ratio $r=v_*/M^2$ and, consequently, the
feasible range of $\rho$ and the resulting power.

\subsection{Block formation rules}

The block partition $\{\mathcal B_r\}_{r=1}^b$ must be
\emph{label-independent} to preserve exact exchangeability under $H_0$.
In our experiments we use two concrete rules:

\begin{itemize}[leftmargin=*]
  \item \textbf{Mean difference statistic.}
  For univariate or low-dimensional mean-difference testing, we form
  blocks by equal-frequency quantiles of the pooled sample values.
  This groups observations with similar magnitude, reducing within-block
  heterogeneity and helping to stabilize the one-swap increments
  $h(Z_j-Z_i)$.

  \item \textbf{$\widehat{\mathrm{MMD}}^2$ statistic.}
  For MMD-based tests, we construct blocks using kernel mean scores
  \[
  s_i = \frac{1}{N}\sum_{j=1}^N k(Z_i,Z_j),
  \]
  and partition the pooled sample into equal-size blocks based on
  $s_i$.
  This encourages swaps between points that differ in their kernel
  neighborhood structure and thus maximizes the impact on the MMD
  statistic.
\end{itemize}

\subsection{Complementary block--pair scheme}

Our empirical design uses a \emph{complementary block--pair}
construction for cross-swaps.
Blocks are ordered by their score (e.g., pooled value quantile or kernel
score), and we pair extremes symmetrically:
$\mathcal B_1$ with $\mathcal B_b$, $\mathcal B_2$ with $\mathcal B_{b-1}$,
and so on.
Swaps are then restricted to pairs of representatives drawn from each
complementary pair.

This pairing strategy is designed to optimize the trade-off between variance control and signal detection:
\begin{itemize}[leftmargin=*]
  \item \textbf{Preserving validity.}
  It maintains exact validity via exchangeability,
  since block assignments depend only on the pooled data and not on
  group labels.

  \item \textbf{Eliminating redundancy.}
  It filters out statistically redundant swaps between similar values.
  Such low-contrast swaps would contribute to the null variance without
  meaningfully enhancing the statistic's sensitivity.

  \item \textbf{Maximizing sensitivity.}
  By pairing blocks with the largest expected contrast,
  it maximizes the signal under the alternative
  and ensures that the reduced variance budget ($O(h^2)$) is strategically utilized.
\end{itemize}

In essence, the complementary design does not merely shrink the reference spread; it improves the \emph{quality} of the permutation distribution. By concentrating the randomization on high-contrast pairs, we achieve a sharper critical value while retaining the capacity to detect shifts under $H_1$.

\subsection{Choice of representative ratio \texorpdfstring{$\rho$}{rho}}

The feasibility condition \eqref{eq:tb-feasible} suggests a natural
lower bound
\[
\rho_{\min}:=\frac{8}{9}\frac{\log(1/\alpha)}{rN},
\]
which guarantees that the variance term dominates the critical value.
In practice, we set
\[
\rho_{\mathrm{opt}}\ \gtrsim\ c\,\rho_{\min}
\qquad\text{with}\quad
c\in[1.2,1.5],
\]
which places the test safely in the variance regime without
unnecessary variance inflation and often yields the highest empirical
power.
All quantities involved in $r$ are data-dependent and computable from
the restricted swap set $\mathcal P$.

\section{Empirical Results: Complementary Block--Pair Simulation}
\label{sec:sim}

We now present finite-sample experiments comparing the classical full
relabeling permutation test with the proposed complementary
block--pair restricted permutation scheme.
Two statistics are tested: the unbiased $\widehat{\mathrm{MMD}}^2$ and
the sample mean difference.
Throughout we use $\rho=0.2$, $\alpha=0.05$, and $M=100$ permutations.

\paragraph{Design summary.}
The complementary block--pair construction restricts cross-swaps to
\emph{quantile-symmetric blocks}:
extreme samples in the lower-scoring block are exchanged only with those
in the highest-scoring block, and so on toward the center.
As discussed in Section~\ref{sec:design}, this pairing strategy
preserves exact validity, reduces redundant swaps, and increases
contrast under the alternative.

\subsection{Simulation setup}

For each replicate, we draw $n_1=n_2=n\in\{32,64,128,256\}$ from two
multivariate Gaussians.
In the $d=10$ MMD case, $Y$ has a mean shift $\mu=(0.4,0,\ldots,0)$;
in the $d=2$ mean-difference case, $\mu=(0.4,0)$.
Each configuration is repeated $N_{\mathrm{sim}}=100$ times.

For the MMD statistic, we use the unbiased
$\widehat{\mathrm{MMD}}^2$ computed with the Gaussian kernel.
When constructing the restricted permutation blocks, the criterion
differs by statistic:
for MMD, the pooled samples are partitioned into equal-size blocks
based on kernel mean scores $s_i$,
while for the mean-difference statistic, blocks are formed by
equal-frequency quantiles of the pooled sample values.
In both cases, the block assignment is label-independent, preserving
exact exchangeability under $H_0$.

\subsection{Type-I error and power}

Table~\ref{tab:sim-summary} summarizes the empirical type-I error and
power at $\alpha=0.05$.
Figure~\ref{fig:sim-sidebyside} shows the corresponding power curves
and type-I error rates as a function of $n$.

\begin{table}[h]
\centering
\caption{\textbf{Empirical type-I error and power at $\alpha=0.05$}.
Each entry is the mean over $100$ runs.}
\label{tab:sim-summary}
\begin{tabular}{cccccccc}
\toprule
Dim. & Test & $n$ & Power(Full) & Power(Block) & Type-I(Full) & Type-I(Block) & \#Blocks \\
\midrule
10 & MMD & 32  & 0.18 & 0.19 & 0.04 & 0.04 & 2 \\
   &     & 64  & 0.24 & 0.31 & 0.05 & 0.03 & 3 \\
   &     & 128 & 0.50 & 0.56 & 0.07 & 0.04 & 4 \\
   &     & 256 & 0.86 & 0.89 & 0.08 & 0.02 & 5 \\
\addlinespace
2  & Mean & 32  & 0.34 & 0.37 & 0.05 & 0.01 & 3 \\
   & diff & 64  & 0.53 & 0.66 & 0.05 & 0.02 & 4 \\
   &      & 128 & 0.63 & 0.87 & 0.10 & 0.03 & 5 \\
   &      & 256 & 0.99 & 0.99 & 0.05 & 0.04 & 6 \\
\bottomrule
\end{tabular}
\end{table}

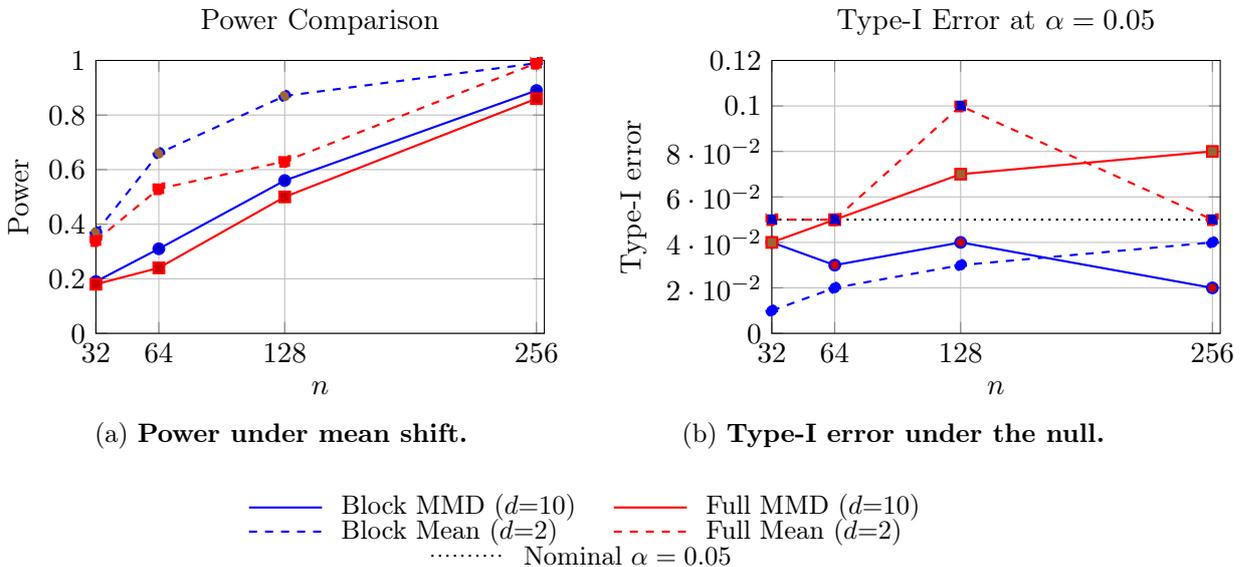
\begin{figure}[t]
\centering
\captionsetup[subfigure]{justification=centering}
\begin{subfigure}[t]{0.47\linewidth}
\centering
\begin{tikzpicture}
\begin{axis}[
    width=\linewidth,
    height=5.2cm,
    xlabel={$n$},
    ylabel={Power},
    xmin=32, xmax=260,
    ymin=0, ymax=1,
    xtick={32,64,128,256},
    ytick={0,0.2,...,1},
    grid=major,
    title={Power Comparison},
    legend style={draw=none},
]
\addplot+[mark=*, thick, blue] coordinates {(32,0.19) (64,0.31) (128,0.56) (256,0.89)};
\addplot+[mark=square*, thick, red] coordinates {(32,0.18) (64,0.24) (128,0.50) (256,0.86)};
\addplot+[mark=*, thick, blue, dashed] coordinates {(32,0.37) (64,0.66) (128,0.87) (256,0.99)};
\addplot+[mark=square*, thick, red, dashed] coordinates {(32,0.34) (64,0.53) (128,0.63) (256,0.99)};
\end{axis}
\end{tikzpicture}
\caption{\textbf{Power under mean shift.}}
\label{fig:sim-power}
\end{subfigure}
\hspace{0.02\linewidth}
\begin{subfigure}[t]{0.47\linewidth}
\centering
\begin{tikzpicture}
\begin{axis}[
    width=\linewidth,
    height=5.2cm,
    xlabel={$n$},
    ylabel={Type-I error},
    xmin=32, xmax=260,
    ymin=0, ymax=0.12,
    xtick={32,64,128,256},
    ytick={0,0.02,...,0.12},
    grid=major,
    title={Type-I Error at $\alpha=0.05$},
    legend style={draw=none},
]
\addplot[domain=32:256, samples=2, thick, black, dotted] {0.05};
\addplot+[mark=*, thick, blue] coordinates {(32,0.04) (64,0.03) (128,0.04) (256,0.02)};
\addplot+[mark=square*, thick, red] coordinates {(32,0.04) (64,0.05) (128,0.07) (256,0.08)};
\addplot+[mark=*, thick, blue, dashed] coordinates {(32,0.01) (64,0.02) (128,0.03) (256,0.04)};
\addplot+[mark=square*, thick, red, dashed] coordinates {(32,0.05) (64,0.05) (128,0.10) (256,0.05)};
\end{axis}
\end{tikzpicture}
\caption{\textbf{Type-I error under the null.}}
\label{fig:sim-type1}
\end{subfigure}

\vspace{4mm}
\begin{tikzpicture}[x=1.2cm,y=0.4cm,baseline]
  \draw[thick, blue] (0,0) -- (0.8,0);
  \node[anchor=west] at (0.9,0) {\small Block MMD ($d{=}10$)};
  \draw[thick, red] (4,0) -- (4.8,0);
  \node[anchor=west] at (4.9,0) {\small Full MMD ($d{=}10$)};
  \draw[thick, blue, dashed] (0,-0.8) -- (0.8,-0.8);
  \node[anchor=west] at (0.9,-0.8) {\small Block Mean ($d{=}2$)};
  \draw[thick, red, dashed] (4,-0.8) -- (4.8,-0.8);
  \node[anchor=west] at (4.9,-0.8) {\small Full Mean ($d{=}2$)};
  \draw[thick, black, dotted] (2,-1.6) -- (2.8,-1.6);
  \node[anchor=west] at (2.9,-1.6) {\small Nominal $\alpha=0.05$};
\end{tikzpicture}

\caption{\textbf{Simulation results of the complementary block--pair
permutation design.}
Left: power improvement relative to classical full relabeling.
Right: type-I error control at $\alpha=0.05$.
The external legend summarizes color/line mapping.}
\label{fig:sim-sidebyside}
\end{figure}

Empirically, the block--restricted scheme achieves uniformly higher (or
comparable) power than full relabeling while maintaining type-I error
close to the nominal level, in line with the theoretical variance
contraction and critical value shrinkage established in
Sections~\ref{sec:variance} and~\ref{sec:power}.

\section{Conclusion}

This work introduces a block--restricted one–swap permutation framework that achieves valid inference with a fully explicit and data–dependent reference distribution.
For canonical statistics such as the mean difference and unbiased MMD, the proposed scheme attains a reference variance that is analytically one order smaller (in $h$) than that of the full relabeling benchmark.
This structural variance contraction translates directly into tighter critical values, smaller minimum detectable effects (MDE), and substantially higher power, without compromising validity or exactness.
All variance and tail parameters used in the critical value and power analysis are exactly computable from the data rather than based on worst–case Lipschitz bounds or asymptotic approximations.
Hence, the proposed test not only provides a provably sharper and more interpretable alternative to the classical permutation test, but also offers a transparent path to implementation, reproducibility, and theoretical analysis of power.

Despite these advantages, the present framework leaves several open questions regarding design optimality.
In particular, the theory does not yet prescribe an optimal strategy for choosing the number and boundaries of blocks, nor for determining the most efficient cross–swap structure between them.
Likewise, the selection of the ratio parameter~$\rho$ currently relies on feasibility and empirical tuning rather than a closed–form optimal rule.
Developing principled criteria for these design choices, possibly through asymptotic efficiency or minimax power analysis, remains an important direction for future research.

While this work focuses on the uniform one–swap rule and its complementary–pair variant, other restricted or adaptive structures—such as weighted block pairing or data–driven block formation—could further optimize the balance between variance control and signal amplification.
Exploring these extensions under the same validity framework would deepen the theoretical understanding and broaden the applicability of restricted permutation methods.

\bibliography{ref}

\end{document}